\documentclass[preprint,12pt]{elsarticle}
\usepackage{amsmath}
\usepackage{amsthm}
 \numberwithin{equation}{section}
\newtheorem{theorem}{Theorem}[section]
\usepackage{graphicx}
\usepackage{epstopdf}
\journal{Nonlinear Analysis: Real World Applications}
\usepackage{fancyhdr}
\begin{document}
\begin{frontmatter}
\title{Estimation of Parameters of an Infectious Disease Model using Neural Networks}
\author[coauth]{V. Sree Hari Rao\corref{cor3} }
\ead{vadrevus@mst.edu}
 \address[coauth]{Department of Mathematics \& Statistics, Missouri University of Science and Technology, Rolla, MO 65409-0020, USA}
\cortext[cor3]{On leave from Jawaharlal Nehru Technological University, Hyderabad - 500 085, India}
\author[focal]{M.~Naresh Kumar\corref{cor1}}
 \ead{nareshkumar\_m@nrsc.gov.in}
 \cortext[cor1]{Principal Corresponding Author}
 \cortext[cor1]{Tel.: +91 40 23884388; Fax.: +91 40 23884437}
 \address[focal]{Software Group, National Remote Sensing Agency (ISRO), Hyderabad, 500037, India}
\cortext[cor3]{Dedicated to Professor V. Lakshmikantham on the occasion of his 85th birthday}
\cortext[cor1]{Copyright $\copyright$ 2010 Elsevier B.V. Accepted for Publication in Nonlinear Analysis: Real World Applications.  DOI: doi:10.1016/j.nonrwa.2009.04.006}
\markboth{\scriptsize Copyright $\copyright$ 2010 Elsevier B.V. Accepted for Publication in Nonlinear Analysis: Real World Applications.  DOI: doi:10.1016/j.nonrwa.2009.04.006} {\scriptsize \thepage}
\begin{abstract}
In this paper, we propose a realistic mathematical model taking into account the mutual  interference among the interacting populations. This model attempts to describe the control (vaccination) function as a function of the number of infective which is an improvement over the existing susceptible-infective epidemic models.
Regarding the growth of the epidemic as a nonlinear phenomenon we have developed a neural network architecture to estimate the vital parameters associated with this model. This architecture is based on a recently developed new class of neural networks known as co-operative and supportive neural networks and it involves a preprocessing of the input data and this renders an efficient estimation of the rate of spread of the epidemic. It is observed that the proposed new neural network outperforms a simple feed forward neural network and polynomial regression.

\end{abstract}

\end{frontmatter}
\section{Introduction}
Simple epidemic models describe the spread of an infectious disease among individuals within the population. After a period, two groups are formed in the population: those who have not acquired the disease but are likely to contract (susceptible population) and those who are infected (infective population), capable of spreading the disease. The fundamental characteristic of the epidemic is that susceptible individuals contract the disease only by getting in contact with the infective individuals. Also the average time constant or the latency time of the disease depends on the nature of the epidemic. Further, cured individuals do not contract the disease again during the same period. Generally, such epidemics will be treated by appropriate vaccination and/or other efforts, which may be viewed as control efforts to contain the spread of disease. The vaccination effort is regarded as a parameter in the mathematical models. The estimation of the rate at which susceptible individuals become invectives is generally a difficult question for the field scientists engaged in this activity. So it is desirable to have a realistic mathematical model that describes the dynamical interactions between these two classes of populations, such as the susceptible and the infective. For a few earlier studies on epidemiological problems we refer the readers to (\cite{2,3,4}, \cite{7, 13, 16,18,19}). This is the starting point for our investigations in this paper and accordingly, we propose a mathematical model to describe the simple dynamics of the interacting groups of the populations.
 
We introduce the nonlinearities in the interacting populations through mutual interference parameters. Our main result provides suitable ranges for these parameters. Regarding the growth of the epidemic as a nonlinear phenomenon, we have developed a neural net work architecture  to estimate the vital parameters of the model. Convergence of the neural network training is an important problem in case of excess data samples being available. We propose a new methodology to train a neural network for data intensive applications by employing clustering techniques.

The present paper is organized as follows. In Section 2, we describe our mathematical model. In Section 3, we derive conditions that ensure the existence and uniqueness of continuable solutions for the model equations. A question of importance to the field scientists is the determination of the rate of spread of the epidemic and we utilize a recently developed neural network architecture [17] to estimate this rate. A related algorithm and the new neural network architecture with K-means clustering are discussed in Section 4. Simulation results are presented in Section 5. The rate of spread of the epidemic is determined using the neural network architecture developed in Section 4 and these results form the content of Section 6.  Conclusions and discussion follow in Section 7.
\section{The Model}
Our model is based on the consideration that the process of the spread of the susceptible population turning into infective population is a nonlinear phenomenon. The following are the underlying biological principles.
\begin{enumerate}
\item	The total population is fixed and initially every individual is susceptible to the disease.
\item	The disease is spread through the direct contact of susceptible individuals with the infective individuals.
\item	Every individual who has contracted the disease and has recovered is regarded as immune.
\end{enumerate}

These principles when translated into the mathematical framework yield the following system of ordinary differential equations and these equations describe the dynamics of the interacting populations, 
\begin{eqnarray}
\dot{x_1} = - \beta x_{1}^{m_1} x_{2}^{m_2} -S(x_1) \\ \nonumber
\dot{x_2}= \beta x_{1}^{m_1} x_{2}^{m_2} - \gamma P(x_2)
\end{eqnarray}
where , $\dot=\frac{d}{dt}$, $x_1$  represents the number of susceptible individuals,  $x_2$ represents the number of infective individuals, $\beta$ is the infection parameter, $\gamma$ denotes a parameter related to the average time constant of the disease. The function $S$ describes the control input and it is assumed to be proportional to the vaccination effort. Also, the function $P$ corresponds to those individuals who have contracted the disease and recovered (regarded as individuals with acquired immunity). Further the functions $S$ and $P$ are assumed to satisfy the following mathematical conditions:
\begin{eqnarray}
S(0) \geq 0, \frac{dS}{dx_1}\geq 0; P(0)=0,  \frac{dP}{dx_2}\geq 0
\end{eqnarray}
The parameters $m_1$   and   $m_2$ appearing in (2.1) represent the indexes of the interacting populations arising out of the non-linear considerations of the epidemic phenomenon. Prototypes of the model (2.1), in which the nonlinear interactions $m_1=m_2=0$  and when a fixed (constant) control effort is invoked, reduce the above model equations to
\begin{eqnarray}
\dot{x_1} = - \beta x_{1} x_{2} -u \\ \nonumber
\dot{x_2}= \beta x_{1} x_{2} - \gamma x_2
\end{eqnarray}
and this model has been studied in  (\cite{1}). More recent studies on the better vaccination efforts may be found in the recent papers (\cite{12, 13, 16,18}).

Clearly the present model (2.1) is more realistic as the control effort essentially varies with regard to the size of susceptible population rather than being fixed. The model (2.3) rests on the simple considerations that if each infected individual   converts one susceptible into infective, then the total number of susceptible population converted into infective population would be  and this describes the interactions between these populations (simple nonlinear interactions). Our model instead of considering simple nonlinear interactions also addresses the sub linear interactions between the two populations.
\section{Existence and Uniqueness of Continuable Solutions}
In this section, we propose to consider the model equations (2.1) and examine the qualitative properties of its solutions. From the biological point of view, the system (2.1) describes the dynamical interactions among the two classes of populations such as the susceptible and the infective populations. Clearly, the qualitative study of solutions of this system depends on ensuring conditions that are sufficient to guarantee the existence of solutions for initial value problems associated with (2.1). Usually this existence of solutions is determined in a finite interval and the solutions are continued on their maximal intervals. This approach yields continuable solutions. Often the inherent dynamics of the system requires one to pick up a specific point on the trajectory and to move in both forward and backward directions. From this discussion, a mathematical treatment of the model equations requires one to obtain conditions for the existence and uniqueness of continuable solutions for the system (2.1). An  analogy with ecological problems would render one to regard the community of infective individuals as a sub-population affecting the survival of the susceptible individuals. One, aspect of the dynamics of community interactions would be mutual interference among the interacting sub-populations, which in our model (2.1) is represented by the parameters  $m_1$ and $m_2$.  It is known  that mutual interference is a ‘stabilizing’ process (\cite{8,9,10}). A question of interest is to describe conditions leading to the persistence/survivability of interacting populations. We observe that the mutual interference introduces sub-linearities into the system (2.1) and it is known that initial value problems for systems with sub-linearities have continuable solutions but these solutions are not unique. (\cite{12,7}), and hence may not be regarded as a dynamical system (\cite{14, 15}). 
Our first result in this direction is to find suitable ranges for parameters  $m_1$ and $m_2$  so that the solutions of (2.1) form a dynamical system in the sense described above
\begin{theorem}
Consider the system of equations given by 
\begin{equation}
\dot{x_i}=g_i(x_1,x_2),
\end{equation}
where $x_i(0) \geq 0$, The functions $g_i: R^{+} \rightarrow R^{+}, R^{+}=[0,\infty)$   are continuous, for $i =1, 2 $,   that is $g_i \in C(R^{+})$. 
\end{theorem}
\begin{proof}
Assume that the following conditions are satisfied: \\
(H1) There exists constants $m_j>0$ such that $h_j \in C(R^{+})$ where $h_j(x_1,x_2)=x_j^{-m_j} g_j(x_1,x_2).$ \\
(H2) $x_k^{m_k}$ $\frac{\partial}{\partial x_k}$ $h_j(x_1, x_2) \in C(R^{+})$ \\

Then the solutions of the system (3.1) form a dynamical system in the sense of (\cite{8}). It is easy to see that the proof is a slight modification of a result of (\cite{4}).  We now apply the content of Theorem 3.1 to our model (2.1). Consider the following transformation of the variables for the system (2.1) given by
\begin{eqnarray}
u_1=x_1^{1-m_1} \\ \nonumber
u_2=x_2^{1-m_2}
\end{eqnarray}
\begin{eqnarray*}
x_1=u_1^{\frac{1}{1-m_1}} \\ 
x_2=u_2^{\frac{1}{1-m_2}}
\end{eqnarray*}
\begin{eqnarray}
x_1^{m_1}=u_1^{\frac{m_1}{1-m_1}} \\ \nonumber
x_2^{m_2}=u_2^{\frac{m_2}{1-m_2}}
\end{eqnarray}
Differentiating equation (3.2) 
\begin{eqnarray}
\dot{u_1}=(1-m_1) x_1^{m_1}\dot{x_1} \implies \dot{x_1}=\frac{1}{1-m_1} x_1^{m_1} \dot{u_1} \\ \nonumber
\dot{u_2}=(1-m_2) x_2^{m_2}\dot{x_2} \implies \dot{x_2}=\frac{1}{1-m_2} x_2^{m_2} \dot{u_2}
\end{eqnarray}
Substituting (3.4) in (2.1) and after some simplifications, we get
\begin{eqnarray}
\mathop {{{\dot u}_1}}\limits^{}  =  - (1 - {m_1})\mathop {}\limits_{} \left[ {\beta \mathop {}\limits_{} {u_2}^{\frac{{m{_2}}}{{1 - m{_2}}}} + {u_1}^{\frac{{ - m{ _1}}}{{1 - m{_1}}}}\mathop {}\limits_{} S({u_1}^{\frac{1}{{1 - m{ _1}}}})} \right] \\ \nonumber
\mathop {{{\dot u}_2}}\limits^{}  = (1 - {m_2})\mathop {}\limits_{} \left[ {\beta \mathop {}\limits_{} {u_1}^{\frac{{m{_1}}}{{1 - m{_1}}}} - \gamma \mathop {}\limits_{} {u_2}^{\frac{{ - m{_2}}}{{1 - m{_2}}}}\mathop {}\limits_{} P({u_2}^{\frac{1}{{1 - m{_2}}}})} \right]
\end{eqnarray}

In order to verify the hypotheses (H1) of Theorem 3.1, we need to show that \[\mathop {\lim }\limits_{{x_j} \to {0^ + }} {h_j}({x_1},{x_2})\]  exists, in which
\[\begin{array}{l}
{h_1}({x_1},{x_2}) = {x_1}^{ - m{_1}}\left[ {\beta \mathop {}\limits_{} {x_1}^{m{  _1}}{x_2}^{m{  _2}} + S({x_1})} \right] = \beta \mathop {}\limits_{} {x_2}^{m{ _2}} + {x_1}^{ - m{  _1}}\mathop {}\limits_{} S({x_1})\\
\\
{h_2}({x_1},{x_2}) = {x_2}^{ - m{  _2}}\left[ {\beta \mathop {}\limits_{} {x_1}^{m{  _1}}\mathop {}\limits_{} {x_2}^{m{  _2}} - \gamma \mathop {}\limits_{} P({x_2})} \right] = \beta \mathop {}\limits_{} {x_1}^{m{  _1}} - \gamma \mathop {}\limits_{} {x_2}^{ - m{  _2}}\mathop {}\limits_{} P({x_2})
\end{array}\]
From the assumptions (2.2) on the functions S and P it is clear that the hypothesis (H1) of Theorem 3.1 is verified. Also, it follows easily that the hypotheses (H2) of Theorem 3.1 are verified provided the parameters $m_k$  satisfies the inequalities \[2m{  _k} - 1 \ge 0\]. 
\end{proof}
Finally, an application of Theorem 3.1 yields the conclusion that solutions of the model (2.1) describe a dynamical system provided \[m{ _k} \ge \frac{1}{2}\] . Henceforth, we designate the inequalities \[m{  _k} \ge \frac{1}{2}\]  as admissible values of the parameters $m_1$ and $m_2$.
\section{Neural Network Architecture}
As mentioned above, a question of practical interest has been to determine the rate at which the susceptible  become  infective. In this section a new neural network architecture with back-propagation algorithm is proposed.

The training data is obtained by solving the model equations (2.1) using Matlab. A simple neural network without a proper preprocessing will not converge due to the non homogeneity in the data sets.  Therefore a preprocessing step with K-means clustering is introduced to create homogeneous data sets before training the network. 
The neural network is trained  using Xlminer Microsoft Excel 2003 plug-in.
\subsection{Methodology}
The k-means clustering works on the expectation of maximization algorithm to find the centers of natural clusters in the data. It assumes that the object attributes form a vector space. The objective is to minimize total intra-cluster variance, or, the squared error function
\[V = {\sum\limits_{i = 1}^k {\sum\limits_{{x_j} \in {S_i}} {\,\,{{\left[ {{x_j} - {\mu _i}} \right]}^2}} } ^{}}\]
where there are $k$ clusters $S_i,  i = 1, 2, \ldots, k$ and $\mu_i$ is the centroid or mean point of all the points $x_j \in S_i$. The initial number of clusters is specified as three based on the visual inspection of the data sets, for generating clusters from the data using k-means clustering. 
The following Figure 1 illustrates the operating  mechanism for training a neural network. The neural network architecture is depicted in  Figure 2.
\begin{figure}[h!]
  \caption{Operating Mechanism forof the proposed Learning paradigm for cooperative networks.}
  \centering
    \includegraphics[width=0.5\textwidth]{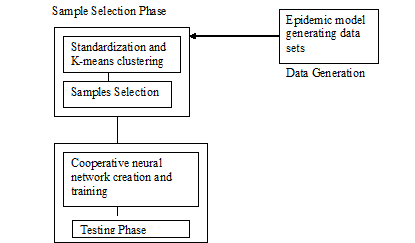}
\end{figure}

\begin{figure}[h!]
  \caption{Cooperative and supportive multi layer feed forward neural network architecture for computing susceptible for a given infective population.}
  \centering
    \includegraphics[width=0.5\textwidth]{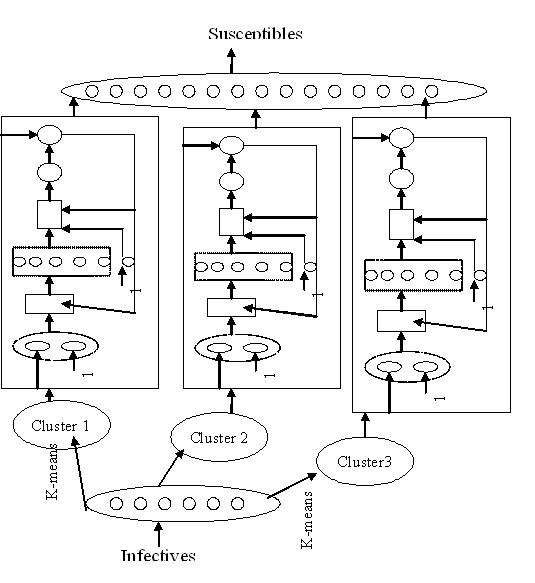}
\end{figure}
K-means clustering is a parametric technique where it is required to provide the number of clusters as a parameter. To know how well-separated the resulting clusters are, a Silhouette plot is constructed using cluster indices output from k-means.
\begin{figure}[h!]
  \caption{Silhouette plot with five, four and three clusters generated using k-means clustering.}
  \centering
    \includegraphics[width=0.5\textwidth]{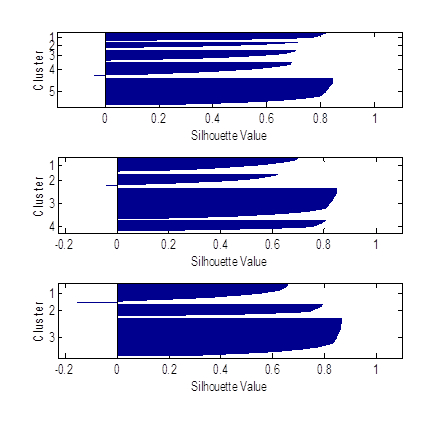}
\end{figure}

The Silhouette plot displays a measure of how close each point in one cluster is to points in the neighbouring clusters. This measure ranges from +1, indicating points that are very distant from neighbouring clusters, through 0, indicating points that are not distinctly in one cluster or another, to -1, indicating points that are probably assigned to the wrong cluster. From the Figure 3 it is derived that the optimal number of clusters is three as silhouette plots of four and five clusters shows very low Silhouette values.

\subsection{Initial Training Procedure}
Initially the entire susceptible and infective populations are standardized to mean zero and standard deviation one.  Using k-means clustering algorithm the data sets are clustered. The set of susceptible and infective populations in each cluster is then partitioned in to training data and testing data. The training data is given as input to the neural network. When the training is complete the testing data is given as input and outputs are computed. 

The three layer neural network architecture with one input, five hidden neurons and one output is considered for training. Each layer has one bias node. The tan hyperbolic function is used as a threshold function for each of the neurons. The network is trained using the back propagation learning algorithm with momentum, which has been found effective. The procedure and the network architecture designed are shown in Figure 4. The data used for training is generated from the mathematical model (2.1) through numerical solutions using Matlab.  

The network is trained to predict the number of susceptible population given the number of infective population to the input neuron. The purpose of predicting susceptible is to estimate the rate of spread of the epidemic. During the training phase the error generated due to differences in the predicted and actual susceptible populations is propagated in the backward direction and weights are adjusted. The network is trained once the mean square error reaches a defined required value. When the training is complete the infective in the testing data set is given as input and the susceptible are computed.  Figure 5 depicts the structure of the trained net work which computes the susceptible population with infective population as input. 

\begin{figure}[h!]
  \caption{Neural network architecture for training.}
  \centering
    \includegraphics[width=0.5\textwidth]{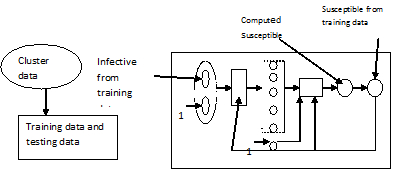}
\end{figure}

\begin{figure}[h!]
  \caption{Testing procedure for the estimation of the susceptible given the infective as in input.}
  \centering
    \includegraphics[width=0.5\textwidth]{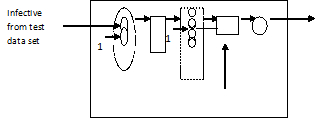}
\end{figure}
The training and testing are done for each of the clusters and the susceptible are computed for the test data sets for all the clusters.

\section{Simulations}
The first part of this section deals with the results of the model (2.1) simulated for various admissible values of the parameters  $m_1$  and $m_2$  .  Model 1 considers a situation in which sub linearity between the susceptible and infective populations is not taken into account and with fixed vaccination effort.

\subsection{Model 1}
This example deals with the case in which $m{ _1} = m{_2} = 1, \quad  \,u = 10,\,\,\,\,\beta  = 0.0001,  \,\quad \gamma  = 0.8,  $ Susceptible population = 10,000, 
and infective population = 10.  This corresponds to the model studied in (\cite{1})
\[\begin{array}{l}
\mathop {\dot x{  _1}}\limits^{}  =  - \beta \mathop {}\limits_{} {x_1}{x_2}^{} - u,\\
\mathop {\dot x{ _2}}\limits^{}  = \beta \mathop {}\limits_{} {x_1}^{}\mathop {}\limits_{} {x_2}^{} - \gamma {x_2}
\end{array}\]
Figure 6 shows the interactions between the susceptible and the infective populations generated using Matlab software.
\begin{figure}[h!]
  \caption{Interactions between susceptible and infective population using Model 1.}
  \centering
    \includegraphics[width=0.5\textwidth]{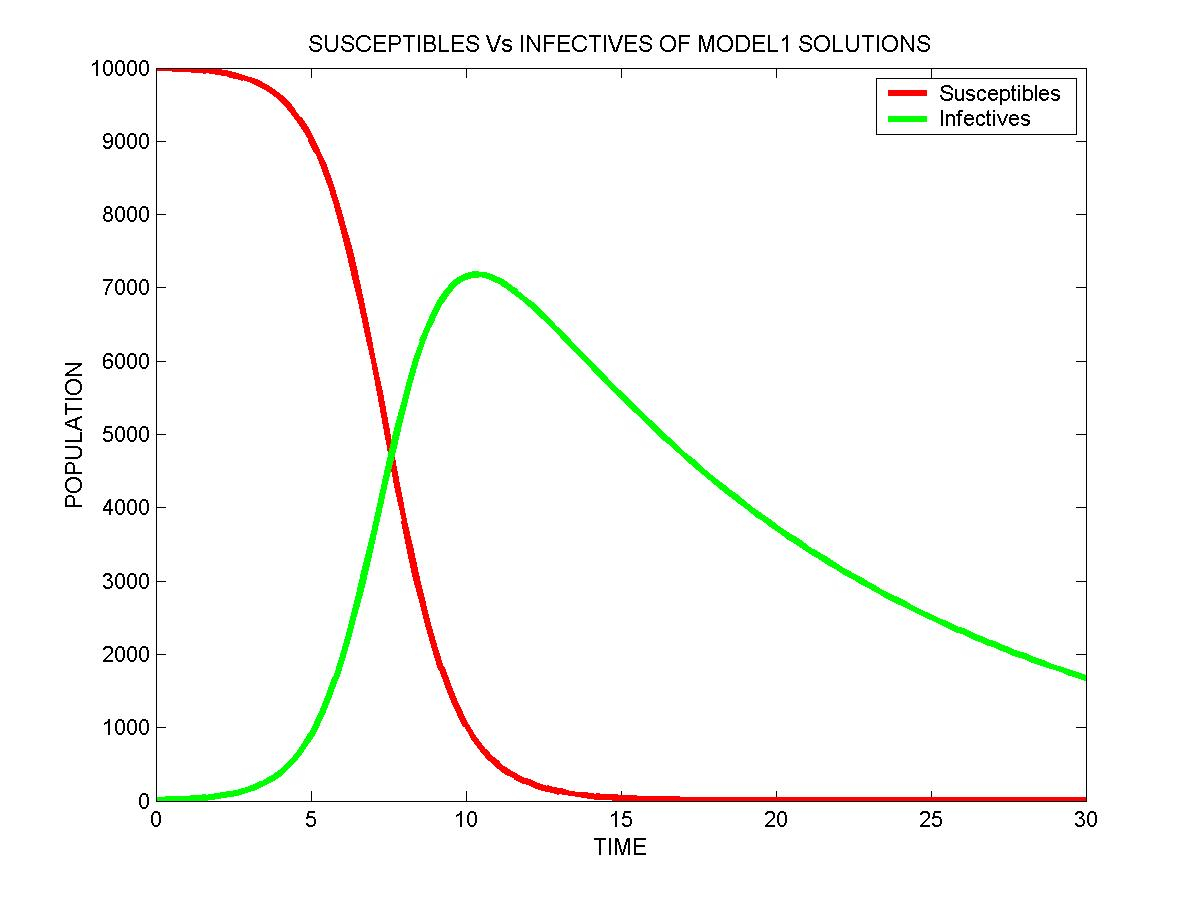}
\end{figure}

\subsection{Comparison of Neural Network and Regression   Analysis for  Model 1}
We have applied statistical methods such as regression analysis for Model 1 and compared the performance of the neural network with the statistical methods.
Neural Network architecture with configuration 1-5-1 is designed.  When k-means clustering is applied to the input data three clusters are obtained. Each individually clustered data is given as input to the three neural networks for training using  back-propagation with momentum learning algorithm. Also a linear and quadratic regression analysis is carried and the results are given in Figure 7.
\begin{figure}[h!]
  \caption{Susceptible individuals computed from neural network and regression analysis giving infective as inputs for Model 1.}
  \centering
    \includegraphics[width=0.5\textwidth]{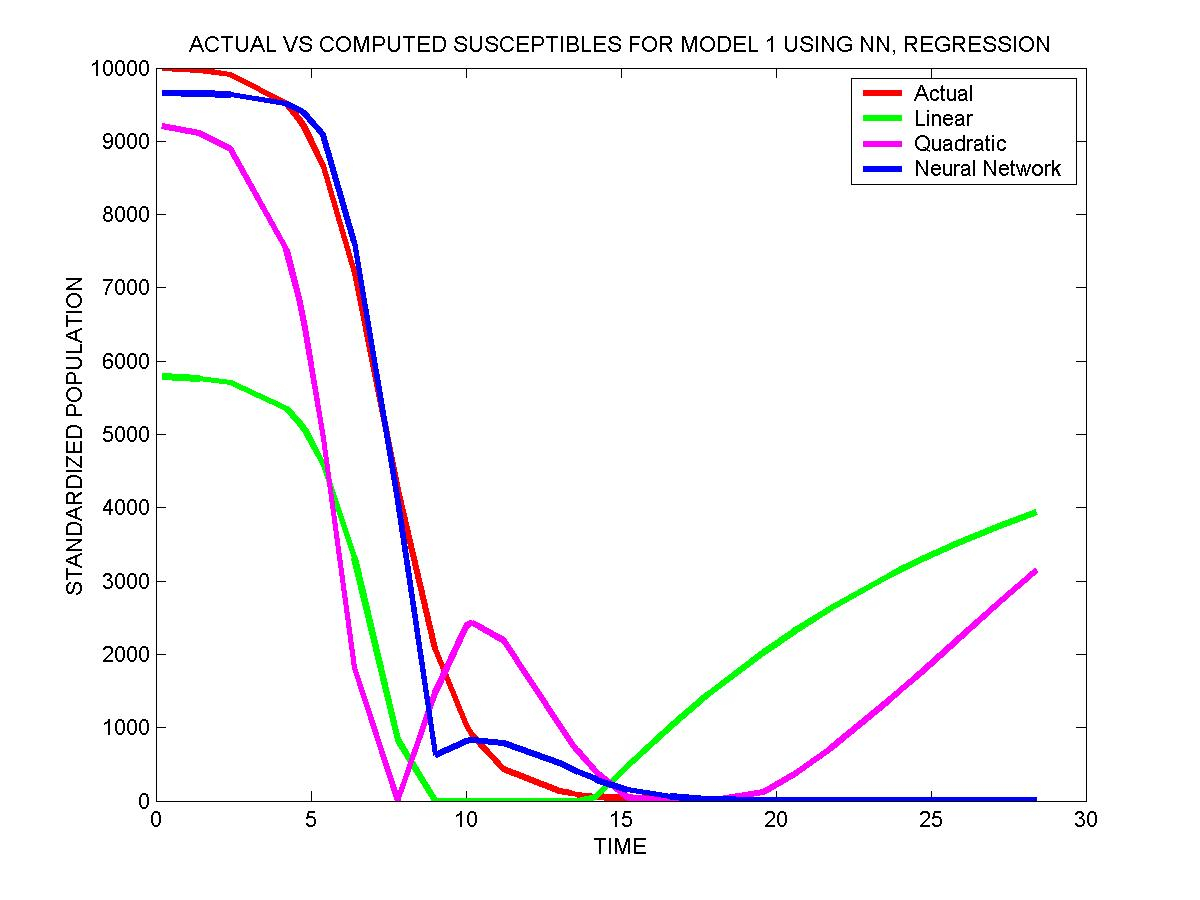}
\end{figure}

Figure 7 shows the actual susceptible and the susceptible estimated from the neural network and regression analysis. It may be seen that the estimation by the  neural network is better when compared with that obtained by using regression method. 

\subsection{Model 2}
\[\begin{array}{l}
\mathop {\dot x{  _1}}\limits^{}  =  - \beta \mathop {}\limits_{} {x_1}^{m{ _1}}\mathop {}\limits_{} {x_2}^{m{  _2}} - S({x_1})\\
\mathop {{{\dot x}_2}}\limits^{}  = \beta \mathop {}\limits_{} {x_1}^{m{\,_1}}\mathop {}\limits_{} {x_2}^{m{  _2}} - \mathop \gamma \limits_{} P({x_2})
\end{array}\]
This model is an improvement over the Model1 in the sense that the vaccination effort is based on the infection rate and is dynamic in character. Also, the interactions between the populations are not necessarily linear. This model has the following parameters \[m{ _{\rm{1}}} = 0.{\rm{8,}}\] $m{ _2} = 0.7,$ \[s(x) = ({x_1}^{0.4}/(vaccination\,effort + {x_1}^{0.4}))\] \[P(x) = {x_1}^{1.2},\] $\beta  = 0.01,\quad \gamma  = 0.04,$ initial susceptible = 1000, and initial infective = 10.

The interactions between the susceptible and infective populations are shown in the Figure 8 which has been generated using Matlab software. The vaccination effort has rendered decline in the infective population well ahead, as observed in Model 1. This clearly establishes that the vaccination effort given by Model 2 is more realistic.
\begin{figure}[h!]
  \caption{Interactions between susceptible and infective individuals during the spread of the epidemic using numerical simulation in Matlab for Model 2.}
  \centering
    \includegraphics[width=0.5\textwidth]{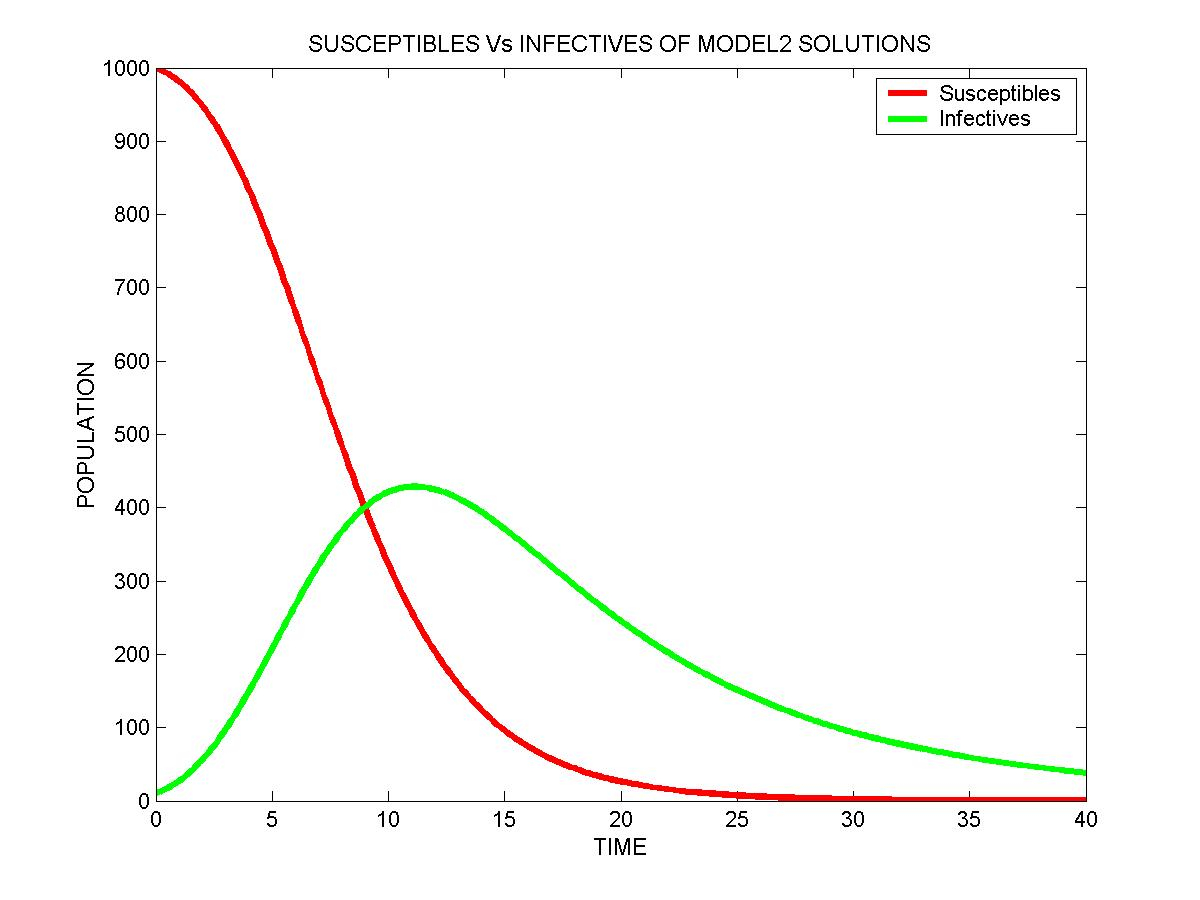}
\end{figure}

\subsection{Neural Network Training with Infective as Input and Susceptible  as expected Output for Model 2}
Neural network architecture similar to the architecture discussed in Section 5.2  is considered for training the network. Linear and quadratic regression analyses are carried out on this data and the results are plotted in the Figure 9. From the Figure 9 it may be observed that the neural network estimation of susceptible is better than those obtained by the regression method.
\begin{figure}[h!]
  \caption{Susceptible individuals computed from neural network and regression analysis with  infective population as input for Model 2.}
  \centering
    \includegraphics[width=0.5\textwidth]{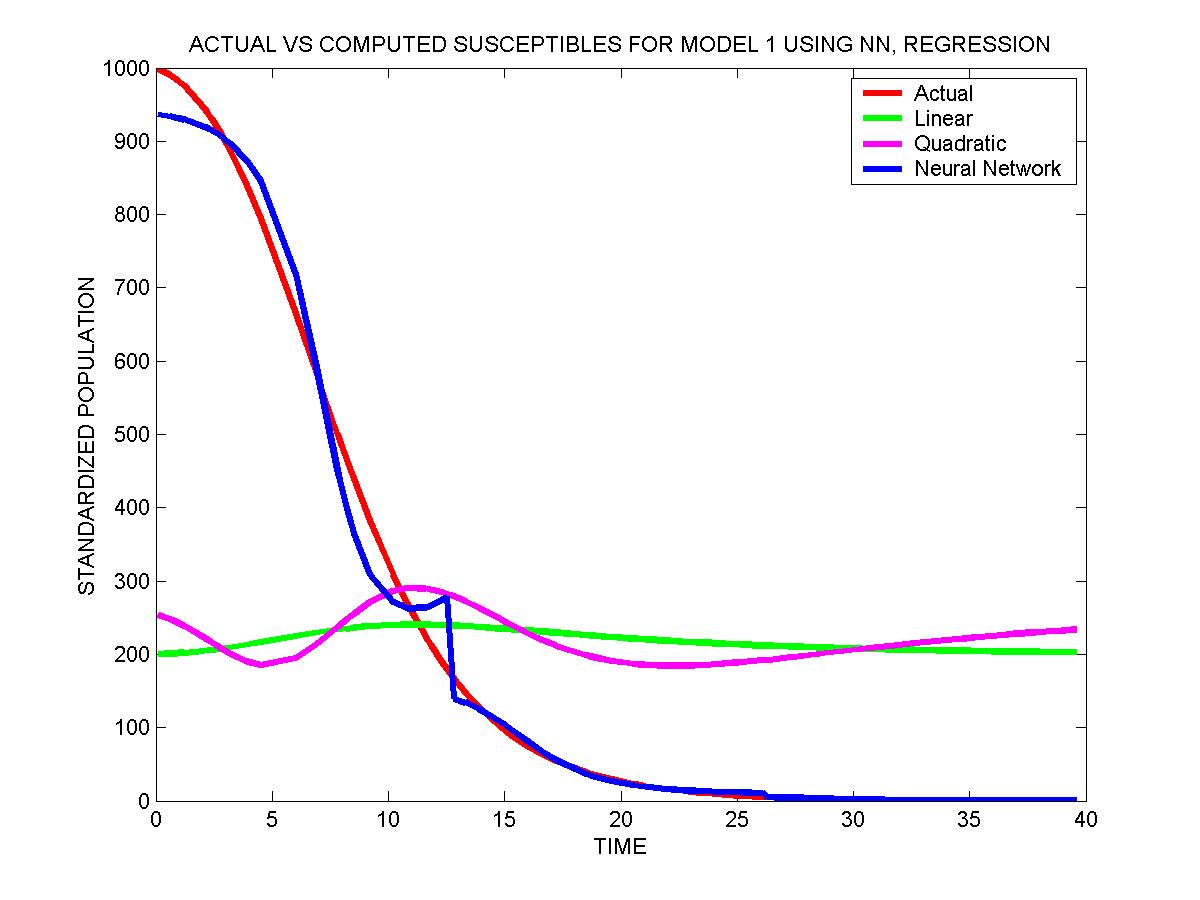}
\end{figure}

\section{Estimation of the rate of Spread}
In the section, we employ the neural network architecture and estimate the rate of spread of the epidemic and compare the same with the actual calculated rate. The rate is the product \[\beta \mathop {}\limits_{} {x_1}^{m{ _1}}\mathop {}\limits_{} {x_2}^{m{ _2}}\].
\subsection{Model 1}
This example shows the rate calculated from a neural network and the actual calculated rate (\[\beta \mathop {}\limits_{} {x_1}^{{m_1}}\mathop {}\limits_{} {x_2}^{{m_2}}\] ) are in good agreement than that obtained from polynomial regression. Simulations are conducted with the following parameter values  $m{ _1} = m{ _2} = 1,\quad \beta  = 0.001.$

The Figure 10 shows that the rate estimated by the neural network is very close to the  actual rate compared to the one estimated by using the regression methods.
\begin{figure}[h!]
  \caption{Estimated rate from neural network, regression analysis and actual  calculated rate of spread of epidemics Model 1.}
  \centering
    \includegraphics[width=0.5\textwidth]{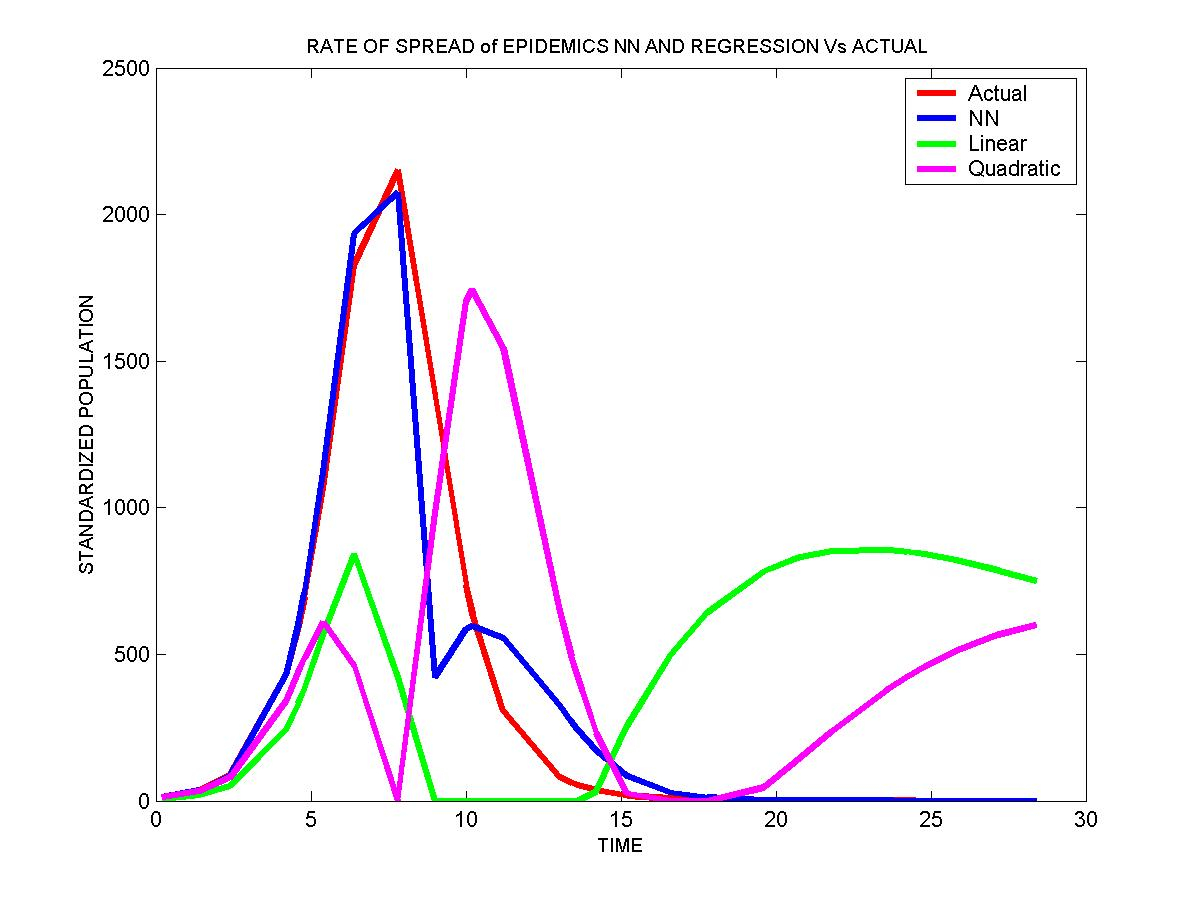}
\end{figure}

\subsection{Model 2}
In this example Model 2 is considered with mutual interference parameters and the rate of spread of the epidemic is calculated  from the neural network and also the polynomial regression method for the following values of the parameters $m{ _1} = 0.8,\,\;m{  _2} = 0.7,\quad \beta  = 0.01.$
\begin{figure}[h!]
  \caption{Estimated rate from neural network, regression analysis and actual calculated rate of spread of epidemics.}
  \centering
    \includegraphics[width=0.5\textwidth]{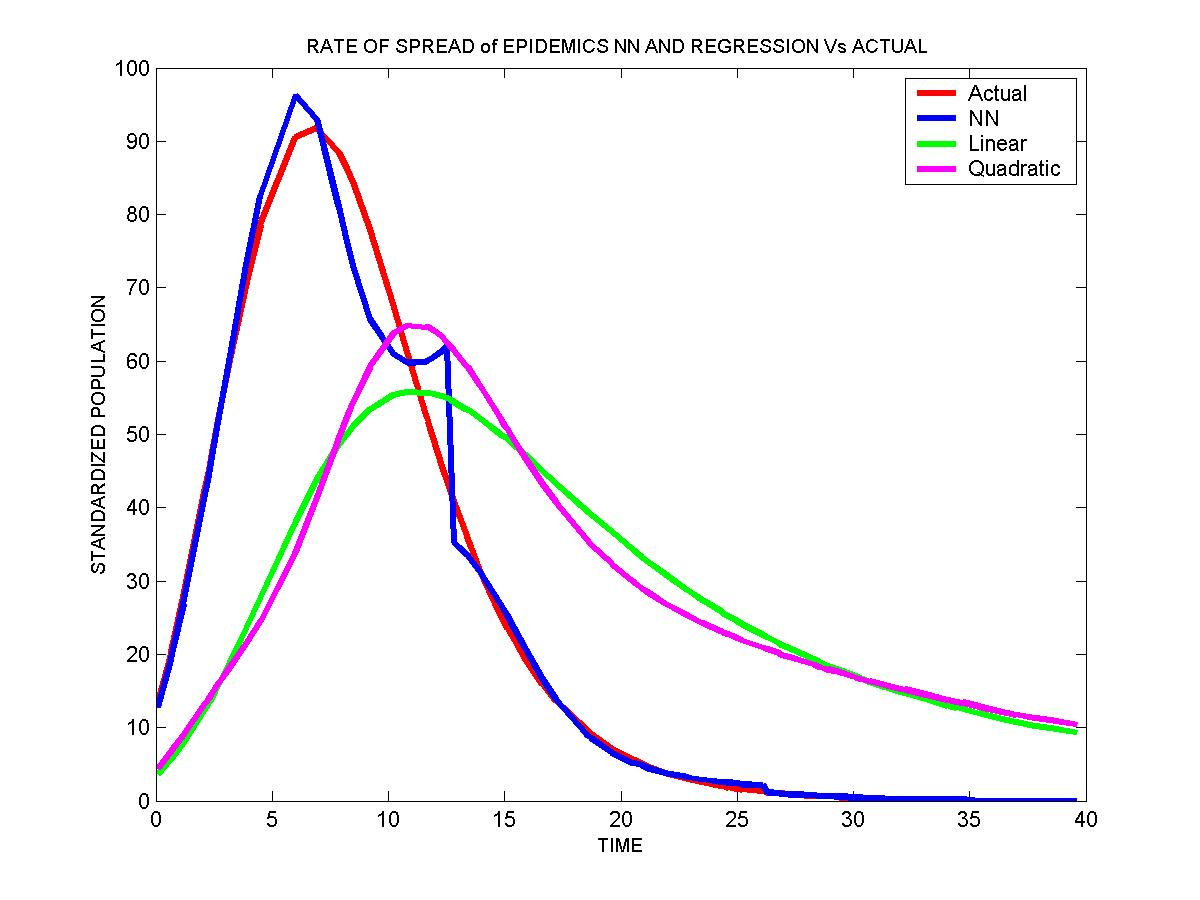}
\end{figure}

In all these examples it is observed that the results obtained through the neural network are closer to the actual results than those obtained by using the polynomial regression analysis.

\section{Conclusions and Discussion}
This work attempts to estimate the rate of spread of an epidemic under realistic conditions, which is a problem of immense concern for the field scientists. It has been realized that a realistic mathematical model is essential to deal with this problem. Accordingly, a mathematical model has been proposed and certain mathematical questions such as the existence and uniqueness of continuable solutions to the model equations (which ensure that the system describes a dynamical system) have been discussed. Following a new learning paradigm using k-means and cooperative neural networks the rate of the spread of epidemic has been estimated. Results of earlier work have been derived from the present work. Simulation results exhibit the decline in both susceptible and infective populations with increased control effort; there by implying that the populations contracting the epidemic are cured and/or the vaccination effort makes them immune. A non-linear regression analysis is also carried out for all the models and it is concluded that the non-linearity in the data is better adapted by the neural network approach than the regression analysis. The neural network out performs the regression analysis by predicting the rate very close to the actual rate. It is hoped that this work paves way for better understanding of the simple epidemic phenomenon.

\bibliographystyle{elsarticle-num}

\end{document}